\documentclass[note,11pt,oneside,bluelines,indentpar]{nrdoc}

\ifx\nrdocument\undefined\relax\else

\reportnumber{SAMBA/20/21}

\publicationyear{2021}


\projectnumber{220865}

\project{Big Insight Explanations}

\target{ML/XAI researchers}

\availability{Open}

\researchfield{XAI}

\fi

\usepackage{amsthm} 
\usepackage{natbib}

\newtheorem{theorem}{Theorem}[section]

\usepackage{bm}
\usepackage{amssymb}

\newtheorem{lemma}{Lemma}[section]
\newtheorem{condition}{Condition}[section]
\newcommand{\indep}{{\perp\!\!\!\perp}}

\newcommand{\newsquare}{{\hfill{\ \vrule height0.5em width0.5em depth-0.0em}}}
\DeclareMathAlphabet\mathbfcal{OMS}{cmsy}{b}{n} 
\DeclareMathAlphabet{\mathcal}{OMS}{cmsy}{m}{n} 

\begin{document}
\title{groupShapley: Efficient prediction explanation with Shapley values for feature groups}
\shorttitle{groupShapley}

\author{Martin Jullum \email{jullum@nr.no} \and 
Annabelle Redelmeier \email{anr@nr.no} \and 
Kjersti Aas \email{kjersti@nr.no}}

\date{Feb 20th, 2021}

\maketitle

\begin{abstract} 
Shapley values has established itself as one of the most appropriate and theoretically sound frameworks for explaining predictions from complex machine learning models.
The popularity of Shapley values in the explanation setting is probably due to its unique theoretical properties.
The main drawback with Shapley values, however, is that its computational complexity grows exponentially in the number of input features, making it unfeasible in many real world situations where there could be hundreds or thousands of features. Furthermore, with many (dependent) features, presenting/visualizing and interpreting the computed Shapley values also becomes challenging.
The present paper introduces \textit{groupShapley}: a conceptually simple approach for dealing with the aforementioned bottlenecks.
The idea is to group the features, for example by type or dependence, and then compute and present Shapley values for these groups instead of for all individual features. Reducing hundreds or thousands of features to half a dozen or so, makes precise computations practically feasible and the presentation and knowledge extraction greatly simplified.
We prove that under certain conditions, \textit{groupShapley} is equivalent to summing the feature-wise Shapley values within each feature group. Moreover, we provide a simulation study exemplifying the differences when these conditions are not met. 
We illustrate the usability of the approach in a real world car insurance example, where \textit{groupShapley} is used to provide simple and intuitive explanations. 
\end{abstract}

\tableofcontents

\mainmatter

\section{Introduction}\label{sec:Intro}
Consider a predictive modelling/machine learning setting with a given model $f(\cdot)$ that takes an $M$ dimensional feature vector $\boldsymbol{x} = \{x_1, \dots, x_{M}\}$ as input and provides a prediction $f(\boldsymbol{x})$ of an unknown response $y$. Suppose that for a specific feature vector, $\boldsymbol{x} = \boldsymbol{x}^*$, we want to understand how the different features (or types of features) contribute to the specific prediction outcome $f(\boldsymbol{x}^*)$.
This task is called prediction explanation and is a type of \textit{local} model explanation, as opposed to a \textit{global} model explanation, which attempts to explain the full model at once, through concepts such as global feature importance \citep[Ch. 2]{molnar2020interpretable}.

Shapley values \citep{Shapley53} has established itself as one of the leading frameworks for prediction explanation. 
The methodology has, in particular, received increased interest following the seminal paper of \citet{Lundberg}. 
The feature-wise Shapley values $\phi_1,\ldots,\phi_M$ for a predictive model $f(\boldsymbol{x})$ at $\boldsymbol{x}=\boldsymbol{x}^*$ are given by
\begin{align}
\phi_j = \sum_{\mathcal{S} \subseteq \mathcal{M} \setminus\{j\}} \frac{|\mathcal{S}| ! (M-| \mathcal{S}| - 1)!}{M!}(v(\mathcal{S}\cup \{j\})-v(\mathcal{S})), \label{eq:shapley}
\end{align}
for $j = 1, \dots, M$, where $M$ is the total number of features, $\mathcal{M}$ is the set of all features, and $v(\cdot)$ is the characteristic function, also referred to as the \textit{contribution function} in the prediction explanation setting. 
Roughly speaking, the information obtained by observing feature $j$ modifies the prediction by the amount of its Shapley value, $\phi_j$.
Although there are different choices available for specifying the contribution function, the conditional expectation 
\begin{align}
v(\mathcal{S}) = \mathbb{E}[f(\bm{x}) | \bm{x}_\mathcal{S} = \bm{x}^*_\mathcal{S}], \label{eq:cond_expec}
\end{align} 
where $\boldsymbol{x}_{\mathcal{S}}$ denotes the subvector of $\boldsymbol{x}$ corresponding to feature subset $\mathcal{S}$, has the nice property of providing explanations which properly account for the dependence between the features \citep{chen2020true}.
This specific contribution function was originally the idea of \citet{Lundberg} and has since been used by several others \citep{aas2019explaining,aas2021explaining,frye2020shapley, redelmeier2020explaining}. 

\subsection{Challenges with the Shapley value framework}
There are, in principle, two challenges when it comes to the practical use of Shapley values for prediction explanation. Since the true (conditional) distribution of $\bm{x}$ is rarely known, the first challenge is estimating \eqref{eq:cond_expec} precisely and efficiently. Poor estimates may lead to severely inaccurate Shapley values and, therefore, also explanations and conclusions \citep{aas2019explaining}. We will not be dealing with this issue in the present paper and, instead, refer to \citet{aas2019explaining, frye2020shapley} for work along these lines.

Instead, this paper concerns itself with the second challenge of handling the fundamental computational complexity of the Shapley formula, which contains $2^{M-1}$ terms and therefore grows exponentially in the number of features $M$. Hence, even with limited time spent on appropriately estimating \eqref{eq:cond_expec}, the computational burden of the arithmetic operations in \eqref{eq:shapley} becomes intractable when there are more than one or two dozen features (depending on the computational resources and time available). 

There exists some model-agnostic approaches for approximating the Shapley value formula. One alternative is to, instead, sum over a smaller number of sampled feature subsets $\mathcal{S}$ \citep{covert2020improving, Lundberg}. However, as the number of features grows, the number of sampled subsets also needs to grow (exponentially) in order to retain an acceptable accuracy for the approximated Shapley values. Thus, this approach only minimally lifts the roof on the number of features that can be handled. In any case, when there are tens, hundreds or even thousands of features, which is not uncommon in the modern machine learning world, these approaches simply do not do. 
\citet{chen2018lshapley} and \citet{li2019efficient} also propose approaches for reducing the complexity of the Shapley value computation. These approaches are very different from the direct sampling approach above in that they assume that the contribution of features to the prediction respects the structure of an underlying graph, and are therefore not generally applicable. There are, however, some conceptual similarities with the method proposed in the present paper. We will return to this in Section \ref{sec:theory}.

When the number of features increases, it also becomes harder to visualize and extract information from Shapley values. This is especially the case when some of the features are (highly) dependent, which in most applied fields is the rule rather than the exception. When features are highly dependent, the joint contribution of these features is distributed among the features, resulting in many small Shapley values. Thus, interpreting the Shapley values and transferring this to practical knowledge requires understanding of the feature dependence structure and appropriate modifications of the interpretation. This clearly becomes an overwhelming task when there are more than a few features. 

Note that some Shapley based prediction explanation approaches \citep{janzing2020feature,sundararajan2020many} replace the contribution function $v(\mathcal{S})$ in \eqref{eq:cond_expec} by a so-called interventional conditional expectation \citep{chen2020true} which is easier to compute and, therefore, reduces the computational burden. However, the computational complexity still grows exponentially in the number of features; it just kicks in a bit later. Furthermore, these methods have the drawback of ignoring the feature dependence, which may lead to inconsistent explanations in the presence of feature dependence. There also exists algorithms for specific model classes, like TreeSHAP \citep{lundberg2020local2global} for tree-based models, which uses the feature dependence learned by the tree(s) to estimate the conditional expectations and efficiently compute Shapley values. However, as seen in \citet{aas2019explaining}, the resulting Shapley values may be highly inaccurate. Since the method is model-specific, it is also difficult to use in a development/testing stage where models of different types are compared.

\subsection{The present paper}
In this paper we propose a novel approach, \textit{groupShapley}, 
to bypass the aforementioned computational issue of explaining predictions using the Shapley value framework. 
The idea is to move away from computing Shapley values for single features and rather compute and present Shapley values for \textit{groups of features}. 
Our definition of grouped Shapley values simply replaces the individual features in \eqref{eq:shapley} by feature groups, creating perfectly well-defined Shapley values with all of the usual Shapley properties. 
The computational complexity is kept small as long as the number of groups is small.
In addition, both the presentation/visualization and the interpretation and knowledge extraction in the presence of high feature dependence, are simplified. 
Consider e.g.~the case when a few base features/data sources are used to construct a myriad of engineered features. Then, \textit{groupShapley} can help improve the interpretability of this model by only presenting Shapley values for the original base features/data sources.


The rest of the paper is organized as follows: Section \ref{sec:groupShapley} provides the mathematical definition of \textit{groupShapley} and discusses how groups can be constructed in different practical settings. 
We discuss and compare \textit{groupShapley} with other ways of establishing a Shapley based measure for the contribution of a group of features. Moreover, we present a theorem saying that under certain conditions, \textit{groupShapley} is equivalent to summing the feature-wise Shapley values within each feature group. Under the same conditions we also derive a simplified formula for the feature-wise Shapley values. In Section \ref{sec:simulations}, we provide a simulation study investigating how \textit{groupShapley} differs from the alternative method under different violations of the conditions 
where they are equal.
In Section \ref{sec:real_data}, we give a real world example using a car insurance data set. Finally, we provide some general discussion and concluding remarks in Section \ref{sec:conclusion}. The supplementary material accompanying the present paper contains proofs of the theoretical results from Section \ref{sec:groupShapley}.

\section{groupShapley}\label{sec:groupShapley}
Recall the predictive modelling/machine learning setting from Section \ref{sec:Intro} with the predictive model $f(\cdot)$ producing predictions based on an $M$-dimensional feature vector $\boldsymbol{x}$ that one wishes to explain. Let us now define a partition\footnote{A partition is a grouping of non-empty subsets of the elements of a set where each element is represented in exactly one group.} $\mathbfcal{G} = \{\mathcal{G}_1,\ldots,\mathcal{G}_G\}$ of the feature set $\mathcal{M}$, consisting of $G$ \textit{groups}. Then, the Shapley value for the $i$-th group $\mathcal{G}_i$ explaining the prediction $f(\boldsymbol{x}^*)$ is given by
\begin{align}
\phi_{\mathcal{G}_i} = \sum_{\mathcal{T} \subseteq \mathbfcal{G} \setminus \mathcal{G}_i} \frac{|\mathcal{T}|_g ! (G-| \mathcal{T}|_g - 1)!}{G!}(v(\mathcal{T}\cup \mathcal{G}_i)-v(\mathcal{T})), \label{eq:pre-grouped} 
\end{align}
where the summation index $\mathcal{T}$ runs over the groups (not the individual features) in the set of groups $\mathbfcal{G} \setminus \mathcal{G}_i$ and $|\mathcal{T}|_g$ refers to the number of groups (not the individual elements) in $\mathcal{T}$. \textit{groupShapley} also adopts the contribution function \eqref{eq:cond_expec}, meaning that 
\begin{align}
    v(\mathcal{T}) = \mathbb{E}[f(\boldsymbol{x})| \boldsymbol{x}_{\mathcal{T}} = \boldsymbol{x}^*_{\mathcal{T}}], \label{eq:vT}
\end{align}
where $\boldsymbol{x}_{\mathcal{T}}$ denotes the subvector of $\boldsymbol{x}$ corresponding to all features contained in all groups in $\mathcal{T}$. Thus, any procedure applicable for estimating \eqref{eq:cond_expec}, can also be directly applied to estimate \eqref{eq:vT}. As \textit{groupShapley} is simply the game theoretic Shapley value framework applied directly to groups of features instead of individual features, the \textit{groupShapley} values possess all the regular Shapley value properties. In particular, the \textit{efficiency property} states that $\sum_{i=1}^G \phi_{\mathcal{G}_i} = f(\boldsymbol{x}^*) - \mathbb{E}[f(\boldsymbol{x})]$, the \textit{symmetry property} roughly states that groups which contribute equally (regardless of the influence of other feature groups) have identical Shapley values, and the \textit{null player property} states that feature groups with no contribution to the prediction (either directly or through other features which it correlates with) have a Shapley value of zero. 

By directly comparing the formula for the feature-wise Shapley values in \eqref{eq:shapley} with the \textit{groupShapley} formula in \eqref{eq:pre-grouped}, we see that the computational complexity of the sum reduces from $2^{M-1}$ to $2^{G-1}$, which gives a relative computational cost reduction of $2^{M-G}$. With, for example, $M=50$ features and $G=5$ feature groups, the relative cost reduction is $>10^{13}$. 
Naturally, in some settings, reducing feature-based explanation to group-based explanation, may not provide a sufficient detail or disguise single feature contributions of interest. However, for large number of features, the benefits of \textit{groupShapley} far outweigh the risk of over-simplification.

\subsection{Defining feature groups}
With the general \textit{groupShapley} procedure established in \eqref{eq:pre-grouped}, the next question is how to define these groups. We propose two alternative strategies: Grouping based on 1) feature type, and 2) feature dependence. These two grouping strategies can be viewed as grouping based on the \textit{practical} properties or the \textit{theoretical} properties of the features, respectively. Note that there is no gold standard for grouping features and it may be informative to compute several explanations of the same predictions using different groupings.

The first strategy groups together features with similar practical meaning, origin or definition.
Consider, for instance, a housing price prediction model with a wide selection of features.
Then, groups made according to feature \textit{type} could be \textit{house attributes} (housing type, square footage, number of rooms, building year), \textit{luxury amenities} (has hot tub/swimming pool/wine cellar/home theatre), \textit{locality} (distance to nearest grocery store/public transport/gym), \textit{environmental/surrounding factors} (area crime rate/air pollution/traffic noise, balcony sunshine hours), and \textit{historical turnover} (previous sold price, details of any refurbishing). As a second example, consider a model that predicts whether a customer purchases a specific item, e.g.~as part of a recommendation system. Then, groups could be specified as follows: \textit{customer details} (gender, age, occupation), \textit{customer purchasing history} (previous purchases, last time item bought), \textit{date/time} (time of day purchases made, day of the week, any upcoming holiday), and \textit{pricing} (item price compared to other brands/stores, current sale, and quantity discount). In these cases, using \textit{groupShapley} may inform that the \textit{locality} of a very central apartment has increased its predicted price, or that it was the \textit{day/time} that reduced the probability of making a specific purchase on a Friday night. In Section \ref{sec:real_data}, we provide a real data example within car insurance where feature type grouping is used. 

The second strategy groups together features based on mutual dependence. The point of this approach is to simplify the interpretability of the computed Shapley values. As mentioned in Section \ref{sec:Intro}, the joint contribution of highly dependent features is spread out on the feature-wise Shapley values making it more difficult to detect the significance of these features. If such highly dependent features are grouped together, it is easier to spot these contributions and, therefore, infer a more `correct' explanation. Such a grouping can be created through a clustering method that uses a dependence/correlation based similarity measure, see e.g.~\citet{kaufman2009finding}. See also the treatment in \citet[Sec. 5]{aas2019explaining}.





\subsection{Other types of grouped Shapley values}
In the general game theoretic setting, using Shapley values to determine the influence of groups of players is not entirely new. \citet{marichal2007axiomatic} defines the concept of a generalized Shapley value for a set of players and studies its theoretical properties. \citet[Ch. 13]{algaba2019handbook} refers to this as the Shapley group value, studies the behavior of the players when acting as a group rather than individually, and applies it to social network `games'.
An important distinction between this methodology and \textit{groupShapley} is that they study how a single group of players acts relative to the other \textit{individual} players, while we compare groups with other groups.  
This makes their method suitable for studying concepts such as which players gain from joining a cooperative group and which are stronger on their own, rather than comparing the relative performance of disjoint groups.
Furthermore, their formulation still has a computational complexity that grows exponentially in the number of players/features. We will, therefore, not devote more space to this method. 







An alternative way to define the contribution of a group of features based on the Shapley value framework is to utilize the additive structure of Shapley values and simply take the sum of the feature-wise Shapley values in \eqref{eq:shapley} for the features within each group. That is, for each group $\mathcal{G}_i$, $i=1,\ldots,G$, we define
\begin{align}
\phi_{\text{post}-\mathcal{G}_i} = \sum_{j \in \mathcal{G}_i} \phi_{j},
\label{eq:post-grouped}
\end{align}
where the summation index $j$ runs over the features in group $\mathcal{G}_i$ and $\phi_j$ is the Shapley value for feature $j$ computed using \eqref{eq:shapley}. 
We call this \textit{post-grouped Shapley} since we explicitly calculate the Shapley values for \textit{individual features} first and then group these Shapley values afterwards. In Section \ref{sec:theory}, we show that under certain conditions, \textit{groupShapley} and \textit{post-grouped Shapley} are identical. However, this is generally not the case, as exemplified through a series of simulation experiments in Section \ref{sec:simulations}. 
Since \textit{groupShapley} provides proper Shapley values when using \eqref{eq:vT} and Shapley values are unique, \textit{post-grouped Shapley} does \textit{not} in general correspond to proper Shapley values using \eqref{eq:vT}.
Moreover,  \eqref{eq:post-grouped} is inherently impractical to compute when there are many features, as it requires the computation of all feature-wise Shapley values -- exactly the bottleneck we are trying to avoid.
That being said, \eqref{eq:post-grouped} is a natural approach to calculating group contributions that has been used before \citep{aas2019explaining, redelmeier2020explaining}. 
Therefore, we will use \textit{post-grouped Shapley} as a reference method when studying the properties of \textit{groupShapley}. Note, however, that there is no reason to prefer \textit{post-grouped Shapley} since it does not provide proper Shapley values and takes longer to compute.

\subsection{Comparing groupShapley with post-grouped Shapley}\label{sec:theory}
In this section we provide two conditions and then show two different theoretical results that hold under the conditions: 1) The feature-wise Shapley values take a simplified form and 2) \textit{groupShapley} and \textit{post-grouped Shapley} are equivalent, i.e. $\phi_{\mathcal{G}_i} = \phi_{\text{post}-\mathcal{G}_i}$, for $i = 1, \dots, G$. We first state the two conditions and then present the results through two theorems. The proofs of the theorems are given in the supplementary material.
\begin{condition}[Partially additively separable]\label{con:additive}
	The predictive function $f(\boldsymbol{x})$ is partially additively separable with respect to the partition $\mathbfcal{G} = \{\mathcal{G}_1,\ldots, \mathcal{G}_G\}$, i.e.~we may write
	\[
	f(\boldsymbol{x}) = \sum_{i = 1}^G f_{\mathcal{G}_i}(\boldsymbol{x}_{\mathcal{G}_i}),
	\]
	where for all $\mathcal{G}_i \in \mathbfcal{G}$, $\boldsymbol{x}_{\mathcal{G}_i}$ is the subset of $\boldsymbol{x}$ corresponding to the features in subset $\mathcal{G}_i$ and $f_{\mathcal{G}_i}$ is a real function involving only $\boldsymbol{x}_{\mathcal{G}_i}$.
\end{condition}

\begin{condition}[Group independence]\label{con:indep}
	All feature groups in partition $\mathbfcal{G}$ are independent ($\boldsymbol{x}_{\mathcal{G}_i} \indep \boldsymbol{x}_{\mathcal{G}_j}$ for all $i \neq j$), meaning that all features in $\mathcal{G}_i$ are mutually independent of all the features in $\mathcal{G}_j$, for all $i\neq j$.
\end{condition}

Before turning to the theorems, let us introduce the contribution function corresponding to the sub-functions from Condition \ref{con:additive}:
$v_{\mathcal{G}_i}(\mathcal{S}) = \mathbb{E}[f_{\mathcal{G}_i}(\boldsymbol{x}_{\mathcal{G}_i}) | \boldsymbol{x}_\mathcal{S} = \boldsymbol{x}_\mathcal{S}^*]$ for any $\mathcal{S} \subseteq \mathcal{G}_i$.

\begin{theorem}[Simplified feature-wise Shapley formula]\label{lem:simplified_shap}
	Assume Conditions \ref{con:additive} and \ref{con:indep} hold.
	Then for any $j \in \mathcal{G}_i$, $i = 1\ldots G$, the Shapley formula in \eqref{eq:shapley} simplifies to
	\begin{align}
    	\phi_j &= \sum_{\mathcal{S} \subseteq \mathcal{G}_i \setminus\{j\}}
    	\frac{|\mathcal{S}|! (|\mathcal{G}_i|-|\mathcal{S}|-1)!}{|\mathcal{G}_i|!} (v_{\mathcal{G}_i}(\mathcal{S} \cup \{j\})-v_{\mathcal{G}_i}(\mathcal{S})), \label{eq:Shapley_indep}
    \end{align}
    i.e.~the $\phi_j, j \in \mathcal{G}_i$ are identical to the Shapley values for explaining the predictive model $f_{\mathcal{G}_i}(\boldsymbol{x}_{\mathcal{G}_i})$.
\end{theorem}
Theorem \ref{lem:simplified_shap} does not only provide a simplified formula for the feature-wise Shapley values which reduces the computation complexity from $2^{M-1}$ to $2^{|\mathcal{G}_i|-1}$ under the stated conditions. It also says that under the stated conditions, the Shapley values for $f(\boldsymbol{x}^*)$ can actually be obtained by computing the Shapley values for $f_{\mathcal{G}_i}(\boldsymbol{x}_{\mathcal{G}_i})$ based only on the $|\mathcal{G}_i|$ features relevant for that function, for $i = 1,\ldots,G$. 
An important special case of Theorem \ref{lem:simplified_shap} is when a single feature $j$ is independent of all the other features and joins the predictive model formula $f(\boldsymbol{x})$ in a purely additive way. Then, $\phi_j = f_j(x_j^*) - \mathbb{E}[f_j(x_j)]$, which generalizes a well known simplification of the Shapley value formula for linear models with independent features, see e.g.~\citet[Section 2.2]{aas2019explaining}.
The formula in \eqref{eq:Shapley_indep} of Theorem \ref{lem:simplified_shap} is conceptually similar to a simplified Shapley value formula in \citet{li2019efficient}. They do, however, work with a different contribution function and rely on an assumed underlying graph structure.


Returning to feature groups, the theorem below states the special case equivalence between \textit{groupShapley} and \textit{post-grouped Shapley}.
\begin{theorem}[groupShapley equivalence]\label{thm:shap_groupShapley}
	Assume Conditions \ref{con:additive} and \ref{con:indep} hold. Then,
	 \begin{align}
         \phi_{\text{post}-\mathcal{G}_i} = \phi_{\mathcal{G}_i} = v_{\mathcal{G}_i}(\mathcal{G}_i)-v_{\mathcal{G}_i}(\emptyset), \quad i = 1,\ldots, G, \notag
     \end{align}        
i.e.~post-grouping in terms of \eqref{eq:post-grouped} is equivalent to groupShapley.
\end{theorem}

Theorem \ref{thm:shap_groupShapley}  does not only say that \textit{groupShapley} and \textit{post-grouped Shapley} are equivalent under the stated conditions, but also that the Shapley values are very easy to compute. 
Conditions  \ref{con:additive} and \ref{con:indep} are, of course, rarely met in practical situations. 
It is important to note that this does not, in any way, invalidate \textit{groupShapley} or \textit{post-grouped Shapley}; it just means that the two approaches do not necessarily always yield identical explanations. 

\section{Simulations}\label{sec:simulations}
We perform a simulation study with the goal of uncovering differences between group values calculated under \textit{groupShapley} and \textit{post-grouped Shapley}. The idea is to start with a model satisfying Conditions \ref{con:additive} and \ref{con:indep} and then show the differences between \textit{groupShapley} and \textit{post-grouped Shapley} as we violate these conditions. We defy the \textit{partially additively separable} condition by including pair-wise linear and non-linear interactions in our models. We defy the \textit{group independence} condition by adding dependence between features that belong to different groups. 
Note that Conditions \ref{con:additive} and \ref{con:indep} are \textit{sufficient}. Thus, if the conditions do not hold, this does \textit{not} necessarily mean that \textit{groupShapley} and \textit{post-grouped Shapley} will yield different results. We explore this in more detail below.

There are three versions of the simulation study: Experiment 1 uses a linear regression model; Experiment 2 uses a generalized additive model; and Experiment 3 uses both of these models and a slightly different feature covariance matrix.
The full parameters of the simulation study are:
\begin{itemize}
    \item \textit{Underlying feature distribution model}: We simulate 10 features with a multivariate Gaussian distribution $p(\bm{x}) = N_{10}(\bm{0}, \Sigma)$. This makes it easy to insert dependence between features. 
    \begin{itemize}
        \item $\Sigma$: The covariance matrix of the joint Gaussian distribution when simulating the features $\boldsymbol{x}$. For Experiments 1 and 2, we set the variance to 1 and the correlation $\rho$ which takes values in $\{0, 0.1, 0.3, 0.7, 0.9 \}$. For Experiment 3, we use a slightly different $\Sigma$ discussed in that section.
    \end{itemize}
    \item \textit{n\_test}: The number of testing observations. Set to \textit{n\_test} = 100.
    \item \textit{Groups}: We put the 10 features into two distinct groupings/partitions. $\mathbfcal{G}_A$ has $G=3$ groups: $\mathcal{G}_1 = \{1, 2, 3, 4\}$, $\mathcal{G}_2 = \{5, 6, 7, 8\}$, and $\mathcal{G}_3 = \{9, 10\}$.  $\mathbfcal{G}_B$ has $G=5$ groups: $\mathcal{G}_1 = \{1, 2\}$, $\mathcal{G}_2 = \{3, 4\}$, $\mathcal{G}_3 = \{5, 6\}$, $\mathcal{G}_4 = \{7, 8\}$, and $\mathcal{G}_5 = \{9, 10\}$.
    \item \textit{Predictive model}: We use two different types of predictive models: a linear regression model (lm) and a generalized additive model (GAM). Both lm and GAM are used to fit three different models: one model without interactions, one model with interactions between features of the \textit{same group}, and one model with interactions between features of \textit{different groups}. We discuss each response function in more detail in Sections \ref{subsec:exper1} and \ref{subsec:exper2}. 
\end{itemize}

Once the features are simulated, we compute the Shapley values of each group using both \eqref{eq:pre-grouped} and \eqref{eq:post-grouped}.
We compare the results between \eqref{eq:pre-grouped} and \eqref{eq:post-grouped} using the \textit{Mean Absolute Deviation} (MAD). 
Assuming $G$ groups, the MAD for individual $i$ is defined as:
\begin{align}\label{eq:MAD}
\text{MAD}_i = \frac{1}{G} \sum_{j=1}^{G} |\phi_{\text{pre}-\mathcal{G}_j}(\bm{x}_i) - \phi_{\text{post}-\mathcal{G}_j}(\bm{x}_i)|,
\end{align}
where $\phi_{\text{pre}-\mathcal{G}_j}(\bm{x}_i)$ and $\phi_{\text{post}-\mathcal{G}_j}(\bm{x}_i)$ denote, respectively, the value estimated with \textit{groupShapley} and the value estimated with \textit{post-grouped Shapley}.
When calculating the \textit{groupShapley} and \textit{post-grouped Shapley} values, we use the true Gaussian distribution to compute \eqref{eq:vT} and \eqref{eq:cond_expec}, since we are not concerned with the estimation of those.
We use Monte Carlo integration with 1000 samples from the true Gaussian distribution to compute the conditional expectations. This introduces some randomness in our empirical results, and is the reason that the MAD is not exactly zero under Conditions \ref{con:additive} and \ref{con:indep}, see Figures \ref{fig:exper1-groupAB}-\ref{fig:exper3-groupAB}. 

To be able to compare the MAD between different experiments
and parameters, the predictive functions specified in the next subsections are standardized 
such that all Shapley value computations are performed on predictive functions with standard deviation 1 (the mean values cancel each other in the MAD formula).

\subsection{Experiment 1: linear regression model}\label{subsec:exper1}
In this experiment, we explore how \textit{groupShapley} and \textit{post-grouped Shapley} vary when using a linear regression model. 
We fit three linear models: 
\begin{itemize}
    \item $lm_1$ uses a simple linear combination of features \textbf{without interactions}: $f_{\text{lm},1}(\bm{x}) = \beta_0 + \sum_{i=1}^{10} \beta_i x_i$, where $\bm{\beta} = \{
      -0.6,   0.2,  -0.8,   1.6,   0.3,  -0.8,  0.5,   0.7,   0.6,  -0.3,   1.5\}$.
    \item $lm_2$ introduces \textbf{interactions between some features of the same group}: $f_{\text{lm},2}(\bm{x}) = f_{\text{lm},1}(\bm{x}) + \gamma_{1} \cdot x_1\cdot x_2 + \gamma_{2} \cdot x_3 \cdot x_4 + \gamma_{3} \cdot x_5 \cdot x_6 + \gamma_{4} \cdot x_7 \cdot x_8 + \gamma_{5} \cdot x_9 \cdot x_{10}$, where $\bm{\gamma} = \{0.4,  -0.6, -2.2,  1.1, 0.0\}$.
    \item $lm_3$ introduces \textbf{interactions between features in different groups}: $f_{\text{lm},3}(\bm{x}) = f_{\text{lm},1}(\bm{x}) + \gamma_{1} \cdot x_1 \cdot x_5 + \gamma_{2} \cdot x_1 \cdot x_7 + \gamma_{3} \cdot x_1 \cdot x_9 + \gamma_{4} \cdot x_3 \cdot x_5 +\gamma_{5} \cdot x_3 \cdot x_7 + \delta_{1} \cdot x_3 \cdot x_9 + \delta_{2} \cdot x_5 \cdot x_9$, where $\bm{\delta} = \{0.1,  0.9\}$ and $\bm{\gamma}$ are as above. 
\end{itemize}
Since models $lm_1$ and $lm_2$ don't have interactions between features of different groups, we can observe strictly what happens when we violate the group independence condition. Note that when we add dependence, for example setting $\rho = 0.3$, this experiment does not distinguish between features of the same group and of different groups. This is extended in Section \ref{subsec:exper3}.

The results of Experiment 1 are in Figure \ref{fig:exper1-groupAB}. Notice that for a simple linear model, interactions between features of different groups are \textit{not enough} to increase the MAD. However, adding stronger dependence between all features (but specifically features of different groups -- see Section \ref{subsec:exper3}), increases the MAD. 
Although this trend is seen for both groupings $\mathbfcal{G}_A$ and $\mathbfcal{G}_B$, the mean, median, and spread is larger for grouping $\mathbfcal{G}_A$ (with only three groups). The reason may be that more features \textit{per group} adds more variance to the \textit{post-grouped Shapley} values. 



\subsection{Experiment 2: generalized additive model}\label{subsec:exper2}
In this experiment, we explore how the difference between \textit{groupShapley} and \textit{post-grouped Shapley} varies when using a GAM. This experiment is analogous to Experiment 1 but with a non-linear base model and (pairwise) non-linear interactions. The non-linear interactions takes the form $h(a, b) = a\cdot b + a\cdot b^2 + b\cdot a^2$. The three models are as follows:
\begin{itemize}
    \item $GAM_1$ uses a simple linear combination of cosines \textbf{without interactions}: $f_{\text{GAM},1}(\bm{x}) = \beta_0 + \sum_{i=1}^{10}  \cos(x_i)$ and $\beta_0 = -0.6$.
    \item $GAM_2$ introduces \textbf{interactions between some features of the same group}.  The entire model is $f_{\text{GAM},2}(\bm{x}) = f_{\text{GAM},1}(\bm{x}) + h(x_1, x_2) + h(x_3, x_4) + h(x_5, x_6) + h(x_7, x_8) + h(x_9, x_{10})$.
    \item $GAM_3$ introduces \textbf{interactions between features in different groups}. The entire model is $f_{\text{GAM},3}(\bm{x}) = f_{\text{GAM},1}(\bm{x}) + h(x_1,x_5) + h(x_1, x_7) + h(x_1, x_9) + h(x_3,  x_5) + h(x_3, x_7) + h(x_3, x_9) + h(x_5, x_9)$. 
\end{itemize}

Notice that these models are similar to those in Experiment 1, except that we use cosines when adding single features to the model and a bivariate function $h(a, b)$ when adding interactions to the model. 
The results are given in Figure \ref{fig:exper2-groupAB}. 
For the models without interactions ($GAM_1$) and with within-group interactions ($GAM_2$), the MAD increases with increasing correlation and otherwise behave fairly similar to their linear counterparts in Experiment 1.
On the other hand, the model with between-group interactions ($GAM_3$) behaves completely different compared to its linear counterpart. It has a large MAD already for independent features, which does not seem to increase further with increasing correlation.

\subsection{Experiment 3: different dependence within and between feature groups}\label{subsec:exper3}

In this experiment, we fix the dependence between features of the \textit{same} group and vary the dependence between features of \textit{different} groups. The covariance matrix $\Sigma$ continues to have 1 on the diagonal. 
However, the correlation for features belonging to the \text{same} group is now fixed to $0.87$, while the correlations between features in different groups takes values in $\{0, 0.1, 0.3, 0.7, 0.9 \}$.
We focus on models $lm_2$ and $GAM_2$ defined above since we are only interested in violating the group independence condition.
The point of this experiment is to investigate the effect of increased between-group correlation, when holding the within-group correlation fixed.

The results for both $lm_2$ and $GAM_2$ are given in Figure \ref{fig:exper3-groupAB}. 
We notice that even if there is strong dependence between features within the same group, small between-group correlation gives a small MAD. 
For grouping $\mathbfcal{G}_A$, any increase in the between-group correlation increases the MAD. For grouping $\mathbfcal{G}_B$, this is not the case and only a strong between-group correlation results in an increase in MAD. 
Like Experiment 1, the MAD of grouping $\mathbfcal{G}_A$ seems to have a larger mean, median, and  spread than grouping $\mathbfcal{G}_B$ across all correlations.
\begin{figure}
    \centering
    \includegraphics[width = 17.4cm, height = 6cm]{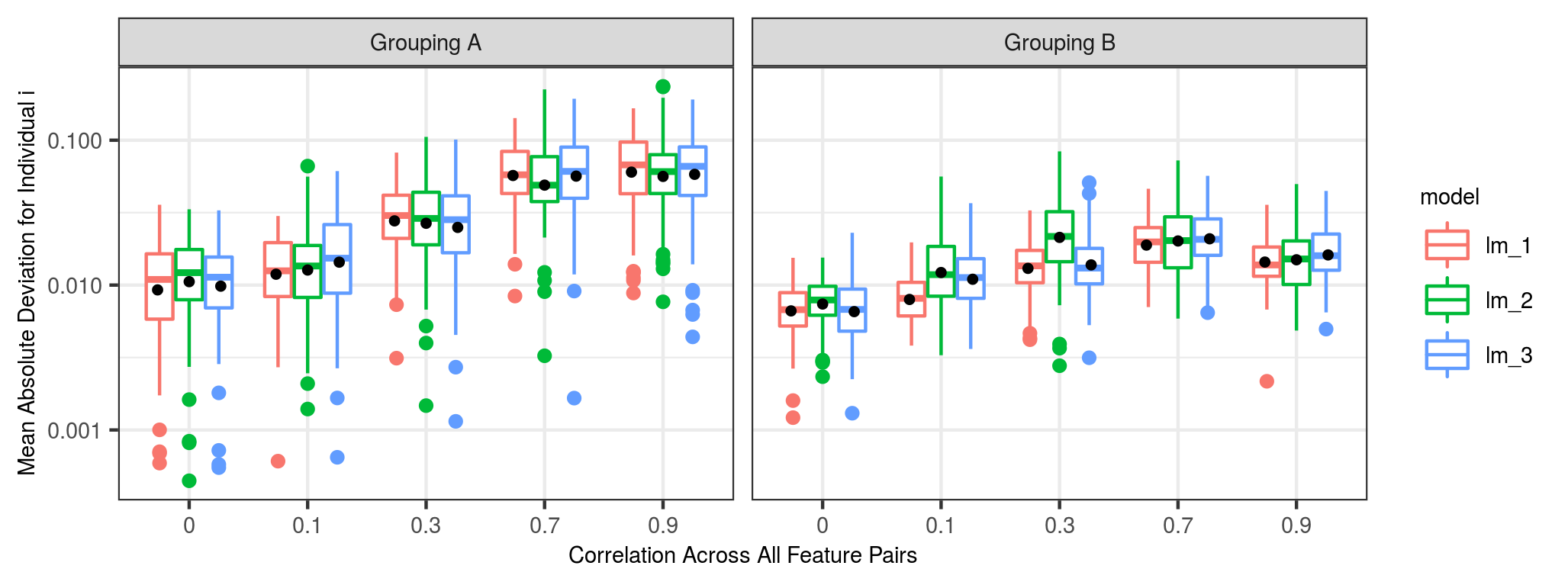}
    \caption{Boxplots of the mean absolute difference on log scale for all three lm models fit in Experiment 1. Black dot indicates the mean.}\label{fig:exper1-groupAB}
\end{figure}
\begin{figure}
    \centering
    \includegraphics[width = 17.4cm, height = 6cm]{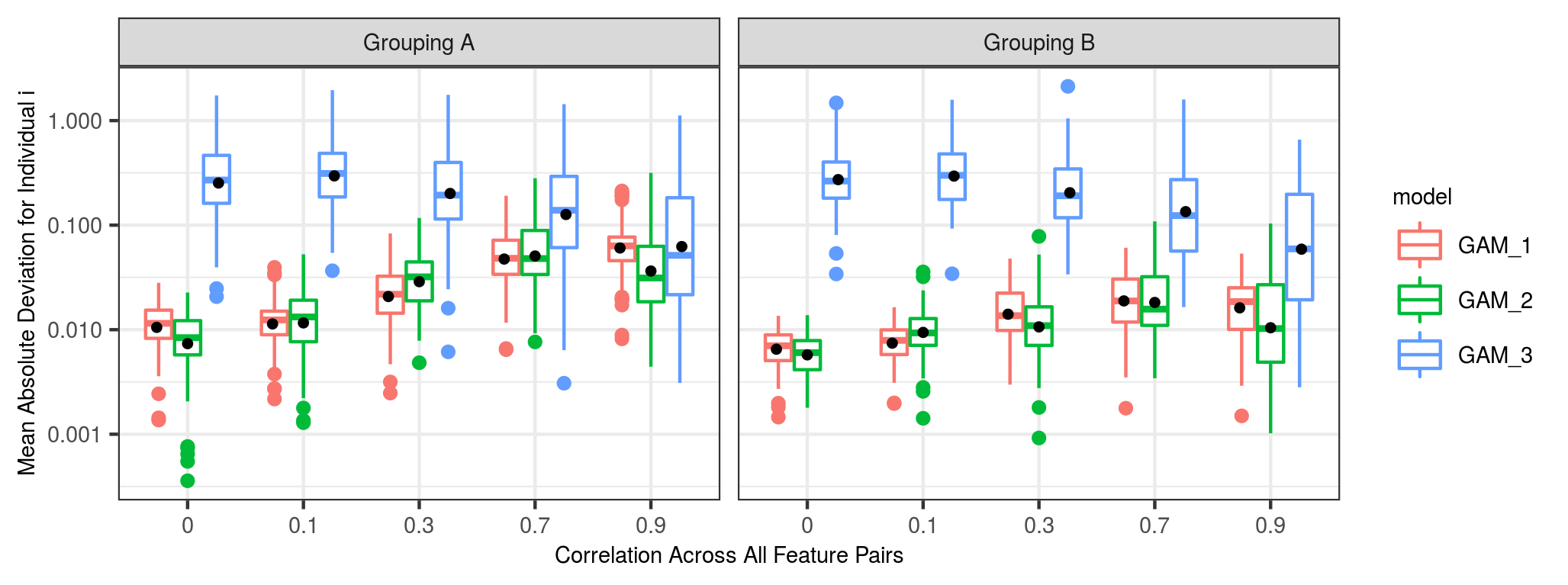}
    \caption{Boxplots of the mean absolute difference on log scale  for all three GAMs fit in Experiment 2. Black dot indicates the mean.}\label{fig:exper2-groupAB}
\end{figure}
\begin{figure}
    \centering
    \includegraphics[width = 17.4cm, height = 6cm]{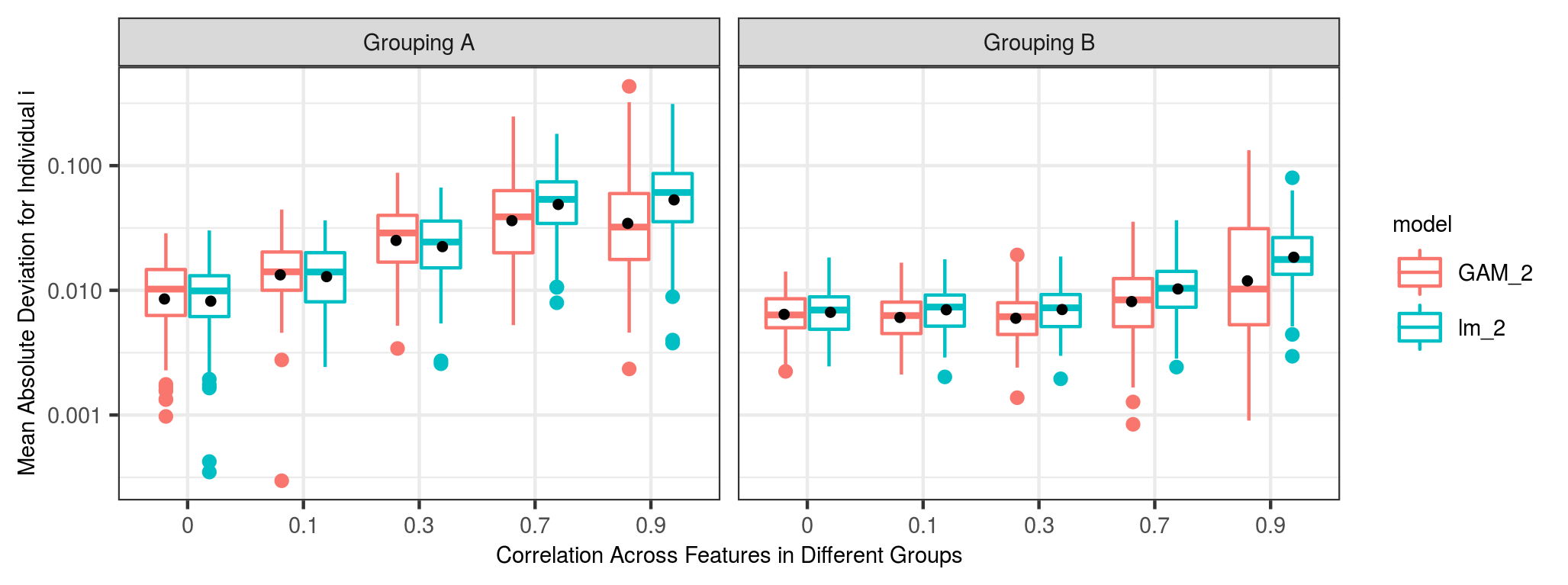}
    \caption{Boxplots of the mean absolute difference on log scale for all $lm_2$ and $GAM_2$ fit in Experiment 3. Black dot indicates the mean. The correlation between features of the same group is fixed at 0.87.}\label{fig:exper3-groupAB}
\end{figure}

The simulation study showed that dependence between features of \textit{different} groups and non-additive models often yield different \textit{groupShapley} and \textit{post-grouped Shapley} values. When using a simple linear model, the MAD increased with the dependence between the features and did not seem to be affected by the linear interactions. For the GAMs, the MAD was large for the model with non-linear interactions between features of \textit{different} groups, but did not increase further with increased dependence. 
Finally, we saw that using a fixed and non-zero between-group correlation did not significantly change the trends seen in Experiments 1 and 2.

\section{Real data example}\label{sec:real_data}
In this section, we demonstrate how \textit{groupShapley} can be applied to a real data set. We use a car insurance data set found on the Kaggle website (\url{https://www.kaggle.com/xiaomengsun/car-insurance-claim-data}). The data contains two different response variables, 23 features, and 10,302 observations. The response we use is the binary variable \textit{customer had a claim}. The features can be naturally partitioned into the following groups:
\begin{itemize}
    \item \textit{Personal Information}: age of driver, highest education level, number of children living at home, value of home, income, job category, number of driving children, martial status, single parent, gender, distance to work, whether driver lives in an urban city, how many years driver has had job.
    \item \textit{Track Record}: number of claims in the past five years, motor vehicle record points, licence revoked in past seven years, amount of time as customer.
    \item \textit{Car Information}: value of car, age of car, type of car, whether car is red.
\end{itemize}

Five of the variables have missing data so we use predictive mean matching to impute these.
To model the probability of a claim, we train a random forest model with 500 trees using the \verb,ranger, R-package \citep{ranger} on the binary response and all 23 features. We use 10 fold cross-validation and get an average out-of-fold AUC of 0.815. The average predicted probability of a claim is 0.273. 
We then use \textit{groupShapley} to calculate the group Shapley values of the \textit{Personal Information}, \textit{Car Information}, and \textit{Track Record} groups for four different individuals. 
Since there is a mix of continuous, discrete and categorical features, the conditional inference tree approach of \citet{redelmeier2020explaining} is used to estimate \eqref{eq:vT}.
We plot the three grouped Shapley values for four different individuals in Figure \ref{fig:real-ex}.

The first individual is a single mother of four (where two children drive). She drives an SUV and drives 27 miles to work. She has had one claim in the last five years and has three motor vehicle record points. 
\textit{Personal Information} gives the largest increase in the predicted probability, which is not surprising given her travel distance and two young drivers.
The second individual is a 37-year-old father of two (where one child drives). He has had one claim in the last five years, his licence revoked in the last seven years, and ten motor vehicle points. His \textit{Track Record} significantly increases his predicted probability, which is natural given his misdemeanors.

The third individual is a 60-year-old married male with no kids at home. He drives a red sports car and has had three claims in the last five years. He has a PhD and currently works as a doctor. His \textit{Personal Information} naturally reduces his predicted probability, while his poor \textit{Track Record}, and to some extent his luxurious \textit{Car Information} increases his predicted probability.
The fourth individual is a 50-year-old female with no kids at home. She drives a minivan and has no previous claims or revoked licences. She also has a PhD and drives 42 miles to work. She appears to be on the safer side of things, which is reflected in all negative \textit{groupShapley} values and a smaller predicted probability.
\begin{figure}
    \centering
    \includegraphics[width = 8cm, height = 7cm]{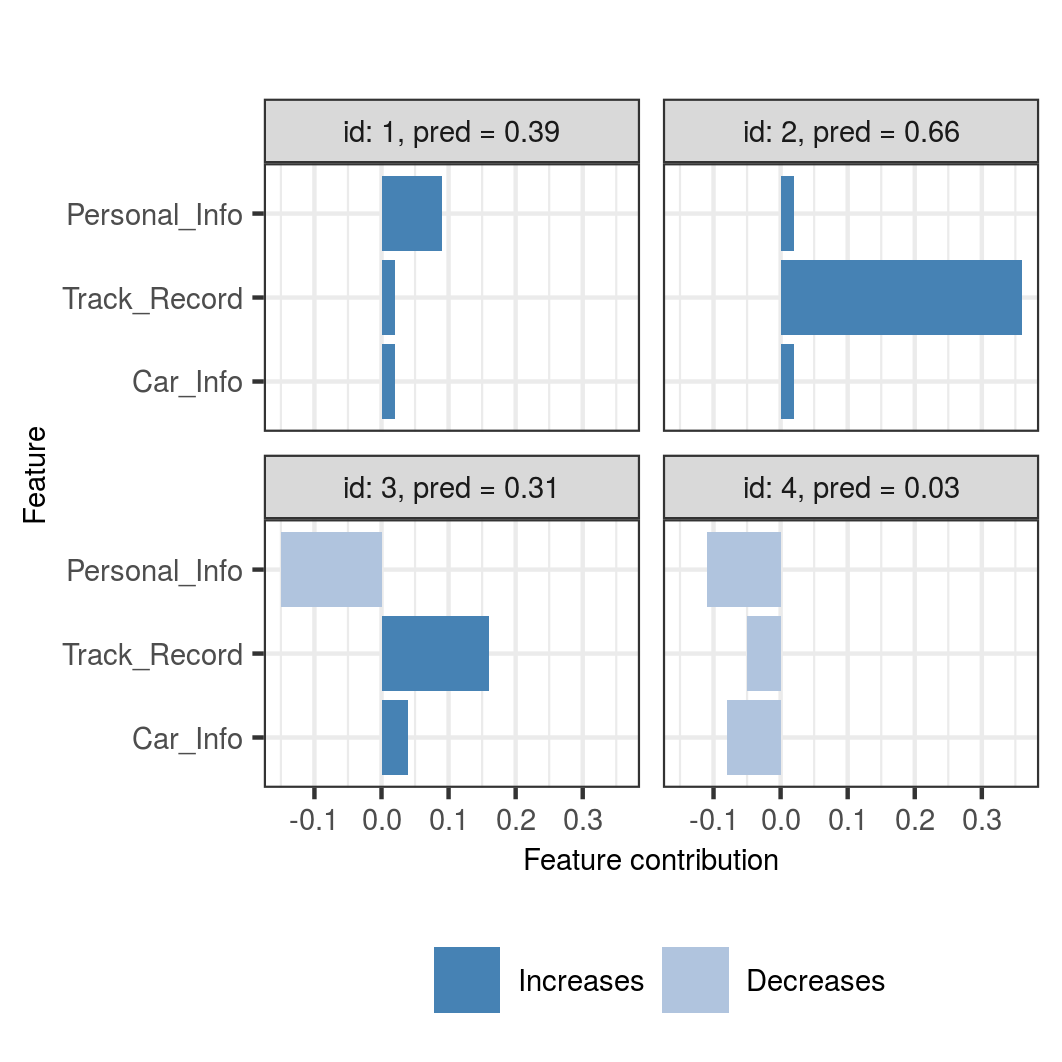}
    \caption{Estimated \textit{groupShapley} values for four individuals with the groups defined in Section \ref{sec:real_data}.}
    \label{fig:real-ex}
\end{figure}

This real world example and Figure \ref{fig:real-ex} exemplifies that \textit{groupShapley} is a simple and efficient way to produce intuitive prediction explanations.

\section{Concluding remarks}\label{sec:conclusion}
In addition to explaining predictions based on feature groups, \textit{groupShapley} can be utilized to compute the relevant feature-wise scores in the case where there are categorical features which are one-hot-encoded when passed to the model $f(\cdot)$. In this case, one is not really interested in the Shapley values for the individual one-hot-encoded features, but for the original features. One can then group the one-hot-encoded features into one group and apply \textit{groupShapley}. This might be viewed as an alternative to the approach in \citet{redelmeier2020explaining} for computing feature-wise Shapley values in the presence of categorical features.

Another possible application of \textit{groupShapley} is to time series classification settings \citep{kvamme2018predicting,ordonez2016deep}. Suppose the model takes one or more time series as input and predicts class probabilities. It may not be interesting or even relevant to compute Shapley values per time point, but by splitting the time series into different chunks, \textit{groupShapley} can be used to explain how different parts of the time series contribute.

It is worth noting that \textit{groupShapley} is a highly general, and computationally efficient solution that can be used in a range of other Shapley value settings. In addition to the prediction explanation setting, it can be reused to bypass the computational burden for \textit{global} explanation methods \citep{gromping2007estimators, Owen17}, or even outside the fields of explainable artificial intelligence and interpretable machine learning \citep{moretti2008transversality}.

Finally, the groupShapley methodology is implemented in the \verb,shapr, R-package \citep{Sellereite2020}\footnote{As of February 2021, the groupShapley methodology is only available in the development version of the package on GitHub: \url{https://github.com/NorskRegnesentral/shapr}.}.

\section*{Ackowledgements}
We thank Anders L{\o}land for useful discussions in the early-stages with this paper and Jens Christian Wahl for discussions related to the simulation study.
This work was supported by the Norwegian Research Council grant 237718 (Big Insight).

\bibliography{bibliography}

\begin{thebibliography}{}

\bibitem[Aas et~al., 2019]{aas2019explaining}
Aas, K., Jullum, M., and L{\o}land, A. (2019).
\newblock Explaining individual predictions when features are dependent: More
  accurate approximations to shapley values.
\newblock {\em arXiv preprint arXiv:1903.10464}.

\bibitem[Aas et~al., 2021]{aas2021explaining}
Aas, K., Nagler, T., Jullum, M., and L{\o}land, A. (2021).
\newblock Explaining predictive models using shapley values and non-parametric
  vine copulas.
\newblock {\em arXiv preprint arXiv:2102.06416}.

\bibitem[Algaba et~al., 2019]{algaba2019handbook}
Algaba, E., Fragnelli, V., and S{\'a}nchez-Soriano, J. (2019).
\newblock {\em Handbook of the Shapley value}.
\newblock CRC Press.

\bibitem[Chen et~al., 2020]{chen2020true}
Chen, H., Janizek, J.~D., Lundberg, S., and Lee, S.-I. (2020).
\newblock True to the model or true to the data?
\newblock {\em arXiv preprint arXiv:2006.16234}.

\bibitem[Chen et~al., 2019]{chen2018lshapley}
Chen, J., Song, L., Wainwright, M.~J., and Jordan, M.~I. (2019).
\newblock L-shapley and c-shapley: Efficient model interpretation for
  structured data.
\newblock In {\em International Conference on Learning Representations}.

\bibitem[Covert and Lee, 2020]{covert2020improving}
Covert, I. and Lee, S.-I. (2020).
\newblock Improving kernelshap: Practical shapley value estimation via linear
  regression.
\newblock {\em arXiv preprint arXiv:2012.01536}.

\bibitem[Frye et~al., 2020]{frye2020shapley}
Frye, C., de~Mijolla, D., Cowton, L., Stanley, M., and Feige, I. (2020).
\newblock Shapley-based explainability on the data manifold.
\newblock {\em arXiv preprint arXiv:2006.01272}.

\bibitem[Gr{\"o}mping, 2007]{gromping2007estimators}
Gr{\"o}mping, U. (2007).
\newblock Estimators of relative importance in linear regression based on
  variance decomposition.
\newblock {\em The American Statistician}, 61(2):139--147.

\bibitem[Janzing et~al., 2020]{janzing2020feature}
Janzing, D., Minorics, L., and Bl{\"o}baum, P. (2020).
\newblock Feature relevance quantification in explainable ai: A causal problem.
\newblock In {\em International Conference on Artificial Intelligence and
  Statistics}, pages 2907--2916.

\bibitem[Kaufman and Rousseeuw, 2009]{kaufman2009finding}
Kaufman, L. and Rousseeuw, P.~J. (2009).
\newblock {\em Finding groups in data: an introduction to cluster analysis},
  volume 344.
\newblock John Wiley \& Sons.

\bibitem[Kvamme et~al., 2018]{kvamme2018predicting}
Kvamme, H., Sellereite, N., Aas, K., and Sjursen, S. (2018).
\newblock Predicting mortgage default using convolutional neural networks.
\newblock {\em Expert Systems with Applications}, 102:207--217.

\bibitem[Li et~al., 2019]{li2019efficient}
Li, X., Dvornek, N.~C., Zhou, Y., Zhuang, J., Ventola, P., and Duncan, J.~S.
  (2019).
\newblock Efficient interpretation of deep learning models using graph
  structure and cooperative game theory: Application to asd biomarker
  discovery.
\newblock In {\em International Conference on Information Processing in Medical
  Imaging}, pages 718--730. Springer.

\bibitem[Lundberg et~al., 2020]{lundberg2020local2global}
Lundberg, S.~M., Erion, G., Chen, H., DeGrave, A., Prutkin, J.~M., Nair, B.,
  Katz, R., Himmelfarb, J., Bansal, N., and Lee, S.-I. (2020).
\newblock From local explanations to global understanding with explainable ai
  for trees.
\newblock {\em Nature Machine Intelligence}, 2(1):2522--5839.

\bibitem[Lundberg and Lee, 2017]{Lundberg}
Lundberg, S.~M. and Lee, S.-I. (2017).
\newblock {A Unified Approach to Interpreting Model Predictions}.
\newblock In {\em {Advances in Neural Information Processing Systems}}, pages
  4768--4777. Curram Associates Inc.

\bibitem[Marichal et~al., 2007]{marichal2007axiomatic}
Marichal, J.-L., Kojadinovic, I., and Fujimoto, K. (2007).
\newblock Axiomatic characterizations of generalized values.
\newblock {\em Discrete Applied Mathematics}, 155(1):26--43.

\bibitem[Molnar, 2020]{molnar2020interpretable}
Molnar, C. (2020).
\newblock {\em Interpretable machine learning: A guide for making black box
  models explainable}.
\newblock Christoph Molnar, Leanpub.

\bibitem[Moretti and Patrone, 2008]{moretti2008transversality}
Moretti, S. and Patrone, F. (2008).
\newblock Transversality of the shapley value.
\newblock {\em Top}, 16(1):1--41.

\bibitem[Ord{\'o}{\~n}ez and Roggen, 2016]{ordonez2016deep}
Ord{\'o}{\~n}ez, F.~J. and Roggen, D. (2016).
\newblock Deep convolutional and lstm recurrent neural networks for multimodal
  wearable activity recognition.
\newblock {\em Sensors}, 16(1):115.

\bibitem[Owen and Prieur, 2017]{Owen17}
Owen, A.~B. and Prieur, C. (2017).
\newblock {On Shapley value for measuring importance of dependent inputs}.
\newblock {\em {SIAM/ASA Journal on Uncertainty Quantification}}, 5:986--1002.

\bibitem[Redelmeier et~al., 2020]{redelmeier2020explaining}
Redelmeier, A., Jullum, M., and Aas, K. (2020).
\newblock Explaining predictive models with mixed features using shapley values
  and conditional inference trees.
\newblock In {\em International Cross-Domain Conference for Machine Learning
  and Knowledge Extraction}, pages 117--137. Springer.

\bibitem[Sellereite and Jullum, 2020]{Sellereite2020}
Sellereite, N. and Jullum, M. (2020).
\newblock shapr: An r-package for explaining machine learning models with
  dependence-aware shapley values.
\newblock {\em Journal of Open Source Software}, 5(46):2027.

\bibitem[Shapley, 1953]{Shapley53}
Shapley, L.~S. (1953).
\newblock {A Value for N-Person Games}.
\newblock {\em {Contributions to the Theory of Games}}, 2:307--317.

\bibitem[Sundararajan and Najmi, 2020]{sundararajan2020many}
Sundararajan, M. and Najmi, A. (2020).
\newblock The many shapley values for model explanation.
\newblock In {\em International Conference on Machine Learning}, pages
  9269--9278. PMLR.

\bibitem[Wright and Ziegler, 2017]{ranger}
Wright, M.~N. and Ziegler, A. (2017).
\newblock {ranger}: A fast implementation of random forests for high
  dimensional data in {C++} and {R}.
\newblock {\em Journal of Statistical Software}, 77(1):1--17.

\end{thebibliography}

\appendix

\section{Proofs}
This appendix contains proofs of Theorem \ref{lem:simplified_shap} and \ref{thm:shap_groupShapley}. Before heading to those proofs, we will present and prove two lemmas that will be useful in those proofs.

\begin{lemma}[Contribution function identity]
\label{cor:vS}
Assume that Conditions \ref{con:additive} and \ref{con:indep} holds. 
Further, let $\mathcal{T}_0$ and $\mathcal{T}_{AB}$ be any two disjoint subsets of the partition $\mathbfcal{G} = \{\mathcal{G}_1,\ldots,\mathcal{G}_G\}$ and $\mathcal{S}_{0}, \mathcal{S}_{A}, \mathcal{S}_{B}$ be any three feature-subsets where $\mathcal{S}_{0} \subseteq \mathcal{T}_0$, $\mathcal{S}_{A}, \subseteq \mathcal{T}_{AB}$, 
$\mathcal{S}_{B}, \subseteq \mathcal{T}_{AB}$. Then, we have that 
\begin{align}
v(\mathcal{S}_0 \cup \mathcal{S}_A) -v(\mathcal{S}_0 \cup \mathcal{S}_B) = v(\mathcal{S}_A) - v(\mathcal{S}_B). \label{eq:vS_diffs}
\end{align}
\end{lemma}
\begin{proof}
Let us write $\boldsymbol{x}_{\mathcal{G}_i}$ for the subvector of $\boldsymbol{x}$ corresponding to the features in group $\mathcal{G}_i$. Let $\mathbfcal{G}_{\mathcal{S}}$ be the set of group indexes with at least one feature in the subset $\mathcal{S}$, i.e.~$\mathbfcal{G}_{\mathcal{S}} = \{i: \mathcal{G}_i \cap S \neq \emptyset\}$, and let $\overline{\mathbfcal{G}}_{\mathcal{S}}$ be  the group indexes not part of $\mathbfcal{G}_{\mathcal{S}}$. To show the equality in \eqref{eq:vS_diffs}, we will first split $v(\mathcal{S})$ into simpler terms by utilizing the additive structure 
$f(\boldsymbol{x}) = \sum_{i = 1}^G f_{\mathcal{G}_i}(\boldsymbol{x}_{\mathcal{G}_i})$ and the group independence $\boldsymbol{x}_{\mathcal{G}_i} \indep \boldsymbol{x}_{\mathcal{G}_j}$ for all $i \neq j$.

We may then utilize Conditions \ref{con:additive} and \ref{con:indep} to decompose $v(\mathcal{S})$ into a sum of conditional expectations involving the groups associated with $\mathcal{S}$ and unconditional expectations for the remaining groups: 
\begin{align}
v(\mathcal{S}) &= \mathbb{E}[f(\boldsymbol{x})|\boldsymbol{x}_{\mathcal{S}}] = \mathbb{E}\left[\sum_{i = 1}^G f_{\mathcal{G}_i}(\boldsymbol{x}_{\mathcal{G}_i}) \bigg| \boldsymbol{x}_{\mathcal{S}}\right] = 
\sum_{i = 1}^G \mathbb{E}\left[ f_{\mathcal{G}_i}(\boldsymbol{x}_{\mathcal{G}_i}) | \boldsymbol{x}_{\mathcal{S}}\right] \notag \\
&=
\sum_{i \in \mathbfcal{G}_{\mathcal{S}}} \mathbb{E}[f_{\mathcal{G}_i}(\boldsymbol{x}_{\mathcal{G}_i}) | \boldsymbol{x}_{\mathcal{S}}] + 
\sum_{i \in \overline{\mathbfcal{G}}_{\mathcal{S}}} \mathbb{E}[f_{\mathcal{G}_i}(\boldsymbol{x}_{\mathcal{G}_i}) | \boldsymbol{x}_{\mathcal{S}}] \notag \\
 & = \sum_{i \in \mathbfcal{G}_{\mathcal{S}}} \mathbb{E}[f_{\mathcal{G}_i}(\boldsymbol{x}_{\mathcal{G}_i}) | \boldsymbol{x}_{\mathbfcal{G}_i \cap \mathcal{S}}] + 
\sum_{i \in \overline{\mathbfcal{G}}_{\mathcal{S}}} \mathbb{E}[f_{\mathcal{G}_i}(\boldsymbol{x}_{\mathcal{G}_i})].
\label{eq:simple-vs}
\end{align}
Now, proving the equality in \eqref{eq:vS_diffs} is just about utilizing the separability of the formula in \eqref{eq:simple-vs} and writing out the four terms:
\begin{align}
v(\mathcal{S}_0 \cup \mathcal{S}_A) & = \sum_{i \in \mathbfcal{G}_{\mathcal{S}_0}} \mathbb{E}[f_{\mathcal{G}_i}(\boldsymbol{x}_{\mathcal{G}_i}) | \boldsymbol{x}_{\mathcal{G}_i \cap \mathcal{S}_0}] + 
\sum_{i \in  \mathbfcal{G}_{\mathcal{S}_A}} \mathbb{E}[f_{\mathcal{G}_i}(\boldsymbol{x}_{\mathcal{G}_i}) | \boldsymbol{x}_{\mathcal{G}_i \cap \mathcal{S}_A}] + 
\sum_{i \in  \overline{\mathbfcal{G}}_{\mathcal{S}_0 \cup \mathcal{S}_A}} \mathbb{E}[f_{\mathcal{G}_i}(\boldsymbol{x}_{\mathcal{G}_i})], \notag \\
v(\mathcal{S}_0 \cup \mathcal{S}_B) & = \sum_{i \in \mathbfcal{G}_{\mathcal{S}_0}} \mathbb{E}[f_{\mathcal{G}_i}(\boldsymbol{x}_{\mathcal{G}_i}) | \boldsymbol{x}_{\mathcal{G}_i \cap \mathcal{S}_0}] + 
\sum_{i \in  \mathbfcal{G}_{\mathcal{S}_B}} \mathbb{E}[f_{\mathcal{G}_i}(\boldsymbol{x}_{\mathcal{G}_i}) | \boldsymbol{x}_{\mathcal{G}_i \cap \mathcal{S}_B}] + 
\sum_{i \in  \overline{\mathbfcal{G}}_{\mathcal{S}_0 \cup \mathcal{S}_B}} \mathbb{E}[f_{\mathcal{G}_i}(\boldsymbol{x}_{\mathcal{G}_i})], \notag \\
v(\mathcal{S}_A) & = \sum_{i \in \mathbfcal{G}_{\mathcal{S}_0}} \mathbb{E}[f_{\mathcal{G}_i}(\boldsymbol{x}_{\mathcal{G}_i})] + 
\sum_{i \in  \mathbfcal{G}_{\mathcal{S}_A}} \mathbb{E}[f_{\mathcal{G}_i}(\boldsymbol{x}_{\mathcal{G}_i}) | \boldsymbol{x}_{\mathcal{G}_i \cap \mathcal{S}_A}] + 
\sum_{i \in  \overline{\mathbfcal{G}}_{\mathcal{S}_0 \cup \mathcal{S}_A}}
\mathbb{E}[f_{\mathcal{G}_i}(\boldsymbol{x}_{\mathcal{G}_i})], \notag \\
v(\mathcal{S}_B) & = \sum_{i \in \mathbfcal{G}_{\mathcal{S}_0}} \mathbb{E}[f_{\mathcal{G}_i}(\boldsymbol{x}_{\mathcal{G}_i})] + 
\sum_{i \in  \mathbfcal{G}_{\mathcal{S}_B}} \mathbb{E}[f_{\mathcal{G}_i}(\boldsymbol{x}_{\mathcal{G}_i}) | \boldsymbol{x}_{\mathcal{G}_i \cap \mathcal{S}_B}] + 
\sum_{i \in  \overline{\mathbfcal{G}}_{\mathcal{S}_0 \cup \mathcal{S}_B}}
\mathbb{E}[f_{\mathcal{G}_i}(\boldsymbol{x}_{\mathcal{G}_i})]. \notag
\end{align}
From the above expressions we easily see that both sides of \eqref{eq:vS_diffs} equal
\begin{align}
&\sum_{i \in  \mathbfcal{G}_{\mathcal{S}_A}} \mathbb{E}[f_{\mathcal{G}_i}(\boldsymbol{x}_{\mathcal{G}_i}) | \boldsymbol{x}_{\mathcal{G}_i \cap \mathcal{S}_A}] - 
\sum_{i \in  \mathbfcal{G}_{\mathcal{S}_B}} \mathbb{E}[f_{\mathcal{G}_i}(\boldsymbol{x}_{\mathcal{G}_i}) | \boldsymbol{x}_{\mathcal{G}_i \cap \mathcal{S}_B}]  \notag \\
+& \sum_{i \in  \overline{\mathbfcal{G}}_{\mathcal{S}_0 \cup \mathcal{S}_A}}
\mathbb{E}[f_{\mathcal{G}_i}(\boldsymbol{x}_{\mathcal{G}_i})] -
\sum_{i \in  \overline{\mathbfcal{G}}_{\mathcal{S}_0 \cup \mathcal{S}_B}}
\mathbb{E}[f_{\mathcal{G}_i}(\boldsymbol{x}_{\mathcal{G}_i})], \notag
\end{align}
and the proof is complete.
\end{proof}

\begin{lemma}[Simplified contribution function identity]
\label{cor:vS2}
Assume that Conditions \ref{con:additive} and \ref{con:indep} holds and that $j \in \mathcal{G}_i$. Further, let $\mathcal{S}$ be a feature subset of $\mathcal{G}_i$ not containing $j$.
Define also $v_{\mathcal{G}_i}(\mathcal{S}) = \mathbb{E}[f_{\mathcal{G}_i}(\boldsymbol{x}_{\mathcal{G}_i}) | \boldsymbol{x}_{\mathcal{S}}]$.
Then we have that
\begin{align}
    v(\mathcal{S} + \{ j \}) - v(\mathcal{S}) =v_{\mathcal{G}_i}(\mathcal{S} + \{ j \}) - v_{\mathcal{G}_i}(\mathcal{S}). \notag
\end{align}
\end{lemma}
\begin{proof}
The result follows almost directly from writing out the expressions for $v$ in terms of the $f_{\mathcal{G}_k}$-functions. Since both $\mathcal{S}$ and $j$ only contain features in $\mathcal{G}_i$, we may, similarly to \eqref{eq:simple-vs}, write
\begin{align}
    v(\mathcal{S}) &= \sum_{k = 1}^G \mathbb{E}\left[ f_{\mathcal{G}_k}(\boldsymbol{x}_{\mathcal{G}_k}) | \boldsymbol{x}_{\mathcal{S}}\right] \notag \\
& = \mathbb{E}[f_{\mathcal{G}_i}(\boldsymbol{x}_{\mathcal{G}_i}) | \boldsymbol{x}_{\mathcal{S}}] + \sum_{k \neq i} \mathbb{E}[f_{\mathcal{G}_k}(\boldsymbol{x}_{\mathcal{G}_k})], \notag \\
    v(\mathcal{S}+\{j\}) &= \sum_{k = 1}^G \mathbb{E}\left[ f_{\mathcal{G}_k}(\boldsymbol{x}_{\mathcal{G}_k}) | \boldsymbol{x}_{\mathcal{S} + \{j\}}\right] \notag \\
& = \mathbb{E}[f_{\mathcal{G}_i}(\boldsymbol{x}_{\mathcal{G}_i}) | \boldsymbol{x}_{\mathcal{S}+ \{j\}}] + \sum_{k \neq i} \mathbb{E}[f_{\mathcal{G}_k}(\boldsymbol{x}_{\mathcal{G}_k})]. \notag
\end{align}
As the sums in the equations cancel each other out, we get
\begin{align}
     v(\mathcal{S} + \{j\}) -  v(\mathcal{S}) &= \mathbb{E}[f_{\mathcal{G}_i}(\boldsymbol{x}_{\mathcal{G}_i}) | \boldsymbol{x}_{\mathcal{S}+ \{j\}}] - \mathbb{E}[f_{\mathcal{G}_i}(\boldsymbol{x}_{\mathcal{G}_i}) | \boldsymbol{x}_{\mathcal{S}}] \notag \\
     &=v_{\mathcal{G}_i}(\mathcal{S} + \{ j \}) - v_{\mathcal{G}_i}(\mathcal{S}), \notag
\end{align}
which completes the proof.
\end{proof}

\subsection{Proof of Theorem 2.1}

We are now going to prove that the two conditions imply that the feature-wise Shapley formula takes the following simplifying form
	\begin{align}
    	\phi_j &= \sum_{\mathcal{S} \subseteq \mathcal{G}_i \setminus\{j\}}
    	\frac{|\mathcal{S}|! (|\mathcal{G}_i|-|\mathcal{S}|-1)!}{|\mathcal{G}_i|!} (v_{\mathcal{G}_i}(\mathcal{S} + \{ j \}) - v_{\mathcal{G}_i}(\mathcal{S})). \notag
    \end{align}
To do that, we will apply Lemma \ref{cor:vS}, perform a series of combinatorial operations, before finally applying Lemma \ref{cor:vS2}.
Recall that $j \in \mathcal{G}_i$. Using $\mathcal{R}$ in place of the standard feature-wise subset notation $\mathcal{S}$, and decompose it through the partition $\mathcal{R}$, i.e.~$\mathcal{S}\cup \mathcal{S}_0 = \mathcal{R}$.
We may then write the feature-wise Shapley value formula as
\begin{align}
\phi_j &= \sum_{\mathcal{R} \subseteq \mathcal{M} \setminus\{j\}} \frac{|\mathcal{R}| ! (M-| \mathcal{R}| - 1)!}{M!}(v(\mathcal{R}\cup \{j\})-v(\mathcal{R})) \notag \\
& =\sum_{\mathcal{S} \subseteq {\mathcal{G}_i} \setminus\{j\}} \sum_{\mathcal{S}_0 \subseteq \mathcal{M} \setminus {\mathcal{G}_i}} \frac{(|\mathcal{S}| + |\mathcal{S}_0|) ! (M-| \mathcal{S}| + |\mathcal{S}_0| - 1)!}{M!}(v(\mathcal{S}_0 \cup \mathcal{S} \cup \{j\})-v(\mathcal{S}_0 \cup \mathcal{S})). \label{eq:shap-group1}
\end{align}
Now, letting $\mathcal{S}_0 = \mathcal{S}_0, \mathcal{T}_0 = \mathcal{M}\setminus {\mathcal{G}_i}$ and $\mathcal{T}_{AB} = {\mathcal{G}_i}, \mathcal{S}_A = \mathcal{S} \cup \{j\}, \mathcal{S}_B = \mathcal{S}$ in Lemma \ref{cor:vS}, we see that 
\[ v(\mathcal{S}_0 \cup \mathcal{S} \cup \{j\})-v(\mathcal{S}_0 \cup \mathcal{S}) = 
v(\mathcal{S} \cup \{j\})-v(\mathcal{S}).\]
As a consequence, we may write \eqref{eq:shap-group1} as 
\begin{align}
\phi_j &= \sum_{\mathcal{S} \subseteq {\mathcal{G}_i} \setminus\{j\}}
(v(\mathcal{S} \cup \{j\})-v(\mathcal{S}))
\sum_{\mathcal{S}_0 \subseteq \mathcal{M} \setminus {\mathcal{G}_i}} \frac{(|\mathcal{S}| + |\mathcal{S}_0|) ! (M-| \mathcal{S}| - |\mathcal{S}_0| - 1)!}{M!}. \label{eq:phij}
\end{align}
Now, let us focus on the inner sum of \eqref{eq:phij}.
We will first rewrite the summand using the binomial coefficient and then rewrite this as a sum over the size of $\mathcal{S}_0$ using combinatorics:
\begin{align}
\sum_{\mathcal{S}_0 \subseteq \mathcal{M} \setminus {\mathcal{G}_i}} \frac{(|\mathcal{S}| + |\mathcal{S}_0|) ! (M-| \mathcal{S}| - |\mathcal{S}_0| - 1)!}{M!} 
= \frac{1}{M}\sum_{\mathcal{S}_0 \subseteq \mathcal{M} \setminus {\mathcal{G}_i}}
\frac{1}{\binom{M-1}{|\mathcal{S}|+|\mathcal{S}_0|}} = \frac{1}{M}\sum_{k=0}^{M-|{\mathcal{G}_i}|} \frac{\binom{M - |{\mathcal{G}_i}|}{k}}{\binom{M - 1}{|\mathcal{S}| + k}}. \label{eq:binomsum}
\end{align}

To further simplify the right hand side of \eqref{eq:binomsum}, we repeatedly utilize the identity
\begin{align}
\int_0^1 z^a(1-z)^b \textrm{d}z = \frac{a!\,b!}{(a+b+1)!}, \notag
\end{align}
for positive numbers $a$ and $b$. This is a well known property of the Beta function, and the integration constant of the Beta distribution.
Let us consider the denominator of the summand in \eqref{eq:binomsum}. Writing $a = |\mathcal{S}| + k$ and $b = M-1-|\mathcal{S}| - k$, we have that
\begin{align}
    \frac{1}{\binom{M - 1}{|\mathcal{S}| + k}} = \frac{1}{\binom{b+a}{a}} = \frac{a!\,b!}{(b+a)!} = (a+b+1)\int_0^1 z^a(1-z)^b \textrm{d}z = M\int_0^1 z^{|\mathcal{S}| + k} (1-z)^{M-1-|\mathcal{S}| - k} \textrm{d}z. \notag
\end{align}
Also utilizing that $\sum_{k=0}^q \binom{q}{k} r^k = (1+r)^k$, we then have that 
\begin{align}
\frac{1}{M}\sum_{k=0}^{M-|\mathcal{G}_i|} \frac{\binom{{M-|\mathcal{G}_i|}}{k}}{\binom{M-1}{{|\mathcal{S}|}+k}} &= \sum_{k=0}^{M-|\mathcal{G}_i|} {{M-|\mathcal{G}_i|} \choose k}\int_0^1 z^{{|\mathcal{S}|}+k}(1-z)^{M-1-{|\mathcal{S}|}-k}dz \notag\\
&= \int_0^1z^{|\mathcal{S}|}(1-z)^{M-1-{|\mathcal{S}|}}\sum_{k=0}^{M-|\mathcal{G}_i|} {{M-|\mathcal{G}_i|} \choose k}\left(\frac{z}{1-z}\right)^k\text{d}z \notag \\
&=\int_0^1z^{|\mathcal{S}|}(1-z)^{M-1-{|\mathcal{S}|}}\left(\frac 1 {1-z}\right)^{{M-|\mathcal{G}_i|}}\text{d}z \notag \\
&=\int_0^1z^{|\mathcal{S}|}(1-z)^{|\mathcal{G}_i|-{|\mathcal{S}|}-1}\text{d}z \notag \\
&=\frac{|\mathcal{S}|!(|\mathcal{G}_i|-|\mathcal{S}|-1)!}{|\mathcal{G}_i|}.  \notag 
\end{align}

Inserting this simplified formula into \eqref{eq:phij} gives  
	\begin{align}
    	\phi_j &= \sum_{\mathcal{S} \subseteq \mathcal{G}_i \setminus\{j\}}
    	\frac{|\mathcal{S}|! (|\mathcal{G}_i|-|\mathcal{S}|-1)!}{|\mathcal{G}_i|!} (v(\mathcal{S} + \{ j \}) - v(\mathcal{S})). \notag
    \end{align}
Since this a sum over subsets only contained in $\mathcal{G}_i$, application of Lemma \ref{cor:vS2} allows us to replace $v(\mathcal{S} + \{ j \}) - v(\mathcal{S})$ by $v_{\mathcal{G}_i}(\mathcal{S} + \{ j \}) - v_{\mathcal{G}_i}(\mathcal{S})$, and thereby completes the proof.~\newsquare

\subsection{Proof of Theorem 2.2}
We are now going to prove that under Conditions \ref{con:additive} and \ref{con:indep}, \textit{groupShapley} is equivalent to \textit{post-grouped Shapley} and they take a simplified form, i.e.~
\[
\sum_{j \in \mathcal{G}_i} \phi_j = \phi_{\mathcal{G}_i} = v_{\mathcal{G}_i}(\mathcal{G}_i)-v_{\mathcal{G}_i}(\emptyset), \quad i = 1,\ldots, G.
\]
We will start by showing that $\sum_{j \in \mathcal{G}_i} \phi_j = v_{\mathcal{G}_i}(\mathcal{G}_i)-v_{\mathcal{G}_i}(\emptyset)$. From Theorem \ref{lem:simplified_shap}, we have that the collection $\{\phi_j:j \in \mathcal{G}_i\}$ corresponds to the Shapley values of $f_{\mathcal{G}_i}(\boldsymbol{x}_{\mathcal{G}_i})$. Thus, by the efficiency axiom, we have that $\sum_{j \in \mathcal{G}_i} \phi_j = v_{\mathcal{G}_i}(\mathcal{G}_i)-v_{\mathcal{G}_i}(\emptyset)$. This completes the first part of the proof.


What remains is to show that under the state conditions, we also have $\phi_{\mathcal{G}_i} = v_{\mathcal{G}_i}(\mathcal{G}_i)-v_{\mathcal{G}_i}(\emptyset)$, for which we will utilize Lemma \ref{cor:vS} once again. Letting $\mathcal{S}_0 = \mathcal{T}, \mathcal{T}_0 = \mathbfcal{G}\setminus{\mathcal{G}_i}$ and $\mathcal{S}_A = \mathcal{T}_{AB} = \mathcal{G}_i, \mathcal{S}_B = \emptyset$, we see that 
\begin{align}
 v(\mathcal{T}\cup {\mathcal{G}_i})-v(\mathcal{T}) = v({\mathcal{G}_i}) - v(\emptyset).\label{eq:vv}
\end{align}
Furthermore, by writing out the two terms on the right hand side of \eqref{eq:vv}, as in the proof of Lemma \ref{cor:vS2}, we have that $v({\mathcal{G}_i}) = v_{\mathcal{G}_i}(\mathcal{G}_i) + \sum_{k \neq i} \mathbb{E}[f_{\mathcal{G}_k}(\boldsymbol{x}_{\mathcal{G}_k})]$ and $v(\emptyset) = v_{\mathcal{G}_i}(\emptyset) + \sum_{k \neq i} \mathbb{E}[f_{\mathcal{G}_k}(\boldsymbol{x}_{\mathcal{G}_k})]$. Thus, we get that $v({\mathcal{G}_i}) - v(\emptyset) = v_{\mathcal{G}_i}(\mathcal{G}_i)-v_{\mathcal{G}_i}(\emptyset)$.
Finally, since the Shapley value weights sum to 1, we have that
\begin{align}
\phi_{{\mathcal{G}_i}} 
&= \sum_{\mathcal{T} \subseteq \mathbfcal{G} \setminus {\mathcal{G}_i}} \frac{|\mathcal{T}| ! (|{\mathcal{G}_i}|-| \mathcal{T}| - 1)!}{|{\mathcal{G}_i}|!}(v(\mathcal{T}\cup {\mathcal{G}_i})-v(\mathcal{T})) \notag \\
&= (v_{\mathcal{G}_i}(\mathcal{G}_i)-v_{\mathcal{G}_i}(\emptyset))\sum_{\mathcal{T} \subseteq \mathbfcal{G} \setminus {\mathcal{G}_i}} \frac{|\mathcal{T}| ! (|{\mathcal{G}_i}|-| \mathcal{T}| - 1)!}{|{\mathcal{G}_i}|!}  \notag \\
&= v_{\mathcal{G}_i}(\mathcal{G}_i)-v_{\mathcal{G}_i}(\emptyset), \notag
\end{align}
which completes the proof.~\newsquare

\end{document}